\documentclass{article}
\usepackage{ijcai17}

\usepackage{times}
\usepackage[small]{caption}

\usepackage[utf8]{inputenc} % allow utf-8 input
\usepackage[T1]{fontenc}    % use 8-bit T1 fonts
\usepackage{hyperref}       % hyperlinks
\usepackage{url}            % simple URL typesetting
\usepackage{booktabs}       % professional-quality tables
\usepackage{amsfonts}       % blackboard math symbols
\usepackage{nicefrac}       % compact symbols for 1/2, etc.
\usepackage{microtype}      % microtypography
\usepackage{graphicx}
\usepackage{subcaption}
\usepackage{float}
\usepackage{amsmath}
\usepackage{amsthm}

\usepackage{natbib}

\usepackage[utf8]{inputenc} % allow utf-8 input
\usepackage[T1]{fontenc}    % use 8-bit T1 fonts
\usepackage{url}            % simple URL typesetting
\usepackage{booktabs}       % professional-quality tables
\usepackage{amsfonts}       % blackboard math symbols
\usepackage{nicefrac}       % compact symbols for 1/2, etc.
\usepackage{microtype}      % microtypography
\usepackage{algorithm}
\usepackage[noend]{algpseudocode}

\usepackage{enumitem}

\newcommand{\makevector}[1]{{\mathbf #1}}

\newcommand{\bvec}{{\makevector{b}}}

\newcommand{\wvec}{{\makevector{w}}}
\newcommand{\xvec}{{\makevector{x}}}
\newcommand{\yvec}{{\makevector{y}}}
\newcommand{\zvec}{{\makevector{z}}}

\newcommand{\zerovec}{{\makevector{0}}}

\newcommand{\thetavec}{\boldsymbol{\theta}}

\newcommand{\RR}{{\mathbb{R}}}
\newcommand{\NN}{{\mathbb{N}}}

\newcommand{\Dmat}{\mathbf{D}}

\newcommand{\Xmat}{\mathbf{X}}
\newcommand{\Ymat}{\mathbf{Y}}

\newcommand{\Ppi}{\mathbf{P}^\pi}
\newcommand{\Rpi}{\mathbf{r}^\pi}
\newcommand{\Vpi}{\mathbf{V}^\pi}

\newcommand{\nsamples}{t}

\newtheorem{theorem}{Theorem}

\newcommand{\inlinervec}[2]{%
  \ensuremath{[\negthinspace\begin{smallmatrix}#1\ #2\end{smallmatrix}]}}

  \newcommand{\defeq}[1]{=}
  \newcommand{\gammarrow}{\overset{\Gamma}{\rightarrow}}

\newcommand{\States}{\mathcal{S}}
\newcommand{\Actions}{\mathcal{A}}
\newcommand{\dpi}{d^\pi}

\newcommand{\vparams}{\mathbf{w}}
\newcommand{\rparams}{\mathbf{B}}

\newcommand{\rep}{\boldsymbol{\phi}}
\newcommand{\repall}{\boldsymbol{\Phi}}
\newcommand{\repvec}{\boldsymbol{\phi}}

\newcommand{\rdim}{k}
\newcommand{\nstates}{|\States|}
\newcommand{\xdim}{d}

\newcommand{\regwgt}{\beta_{\phi}}

\newcommand{\fwgt}{1}
\newcommand{\regwgtb}{\beta_{B}}
\newcommand{\regwgtw}{\beta_{w}}
\newcommand{\eye}{\mathbf{I}}

\newcommand{\scalg}{SCoPE}

\newcommand{\triplet}{(\repall, \rparams,\vparams)}
\newcommand{\Xnew}{\Xmat_{\textrm{\tiny new}}}
\newcommand{\repnew}{\repall_{\textrm{\tiny new}}}
\newcommand{\nsamplesnew}{t_{\textrm{\tiny new}}}

\newcommand{\prev}{a}
\newcommand{\curr}{b}
\newcommand{\nex}{c}
\newcommand{\Amat}{\mathbf{A}}

\newcommand{\MC}{Mountain Car}
\newcommand{\PW}{Puddle World}
\newcommand{\AC}{Acrobot}

\title{Learning Sparse Representations in Reinforcement Learning with Sparse Coding}
\author{
Lei Le\thanks{These authors contributed equally.}\\
Dept. of Computer Science \\
Indiana University \\
Bloomington, IN, USA\\
leile@indiana.edu
\And
Raksha Kumaraswamy$^*$\\
Dept. of Computer Science \\
Indiana University \\
Bloomington, IN, USA\\
rakkumar@indiana.edu
\And
Martha White\\
Dept. of Computer Science \\
Indiana University \\
Bloomington, IN, USA \\
martha@indiana.edu
}

\begin{document}

\maketitle

\begin{abstract}
%In this work, we investigate how to effectively use sparse coding for 
A variety of representation learning approaches have been investigated for
reinforcement learning; much less attention, however, has been given to investigating the
utility of sparse coding. 
%Value function estimation within
%reinforcement learning relies on representations of observations,
%to provide an effective basis for learning. 
%is a
%general formalism for sequential decision-making,
%with recent algorithm development
%focusing on function approximation
%to handle large state spaces and high-dimensional, high-velocity (sensor) data.
%%with the rise in data collection and increased sensor information
%%driving algorithm development towards handling high dimensional,
%%high-velocity data. 
%The success of function approximators, however, hinges on the quality of the
%data representation. 
Outside of reinforcement learning, sparse coding representations
have been widely used, with non-convex objectives that result in discriminative representations.
In this work, we develop a supervised sparse coding objective
for policy evaluation. Despite the non-convexity of this objective,
we prove that all local minima are global minima,
making the approach amenable to simple optimization strategies. 
We empirically show that it is key to use a supervised objective, rather than the more
straightforward unsupervised sparse coding approach. 
We compare the learned representations to a canonical fixed sparse representation, called tile-coding,
demonstrating that the sparse coding representation outperforms
a wide variety of tile-coding representations. 
\end{abstract}

\section{Introduction}

For tasks with large state or action spaces, 
where tabular representations are not feasible,
reinforcement learning algorithms
typically rely on 
function approximation.
Whether they are learning the value function, policy or models,
the success of function approximation techniques
hinges on the quality of the representation. Typically,
representations are hand-crafted,
with some common representations including tile-coding,
radial basis functions, polynomial basis functions and Fourier basis functions ~\citep{sutton1996generalization,konidaris2011value}. 
Automating feature discovery, however, alleviates this burden and
has the potential to significantly improve learning.

Representation learning techniques in reinforcement learning 
have typically drawn on the large literature in unsupervised and supervised learning.
%As in these areas, 
Common approaches include
feature selection, including $\ell_1$ regularization on the value function parameters \citep{loth2007sparse,kolter2009regularization,nguyen2013online}
and matching pursuit \citep{parr2008analysis,wakefield2012greedy};
basis-function adaptation approaches \citep{menachie2005basis,whiteson2007adaptive};
instance-based approaches, such as locally weighted regression~\citep{atkeson2003nonparametric},
sparse distributed memories~\citep{ratitch2004sparse},
proto-value functions~\citep{mahadevan2007proto}
and manifold learning techniques \citep{mahadevan2009learning};
and neural network approaches,
including more standard feedforward neural networks \citep{coulom2002reinforcement,riedmiller2005neural,mnih2015human}
as well as random representations~\citep{sutton1993online},
linear threshold unit search~\citep{mahmood2013representation},
and evolutionary algorithms like NEAT~\citep{stanley2002efficient}.

Surprisingly, however, there has been little investigation into using sparse coding
for reinforcement learning.
%There has been some work towards this aim, with 
Sparse coding approaches have been
developed to learn MDP models for transfer learning \citep{ammar2012reinforcement};
outside this work, however, little has been explored.
Nonetheless, such sparse coding representations have several advantages,
including that they naturally enable local models, are computationally efficient to use,
are much simpler to train than more complicated models such as neural networks
and are biologically motivated by the observed representation in the mammalian cortex \citep{olshausen1997sparse}.

In this work, we develop a principled sparse coding objective for policy evaluation.
In particular, we formulate a joint optimization over the basis and
the value function parameters, to provide a supervised sparse coding
objective where the basis is informed by its utility for prediction.
We highlight the importance of using the Bellman error or mean-squared return error
for this objective, and discuss how the projected Bellman error is not suitable. 
We then show that,
%a surprising but important result that, 
despite being a nonconvex objective,
all local minima are global minima, under minimal conditions. We avoid the
need for careful initialization strategies needed for previous optimality results for sparse coding 
\citep{agarwal2014learning,arora2015simple},
%\citep{spielman2012exact,agarwal2014learning,arora2015simple,agarwal2017aclustering},
using recent results for more general dictionary learning settings \citep{haeffele2015global,le2016global}, particularly by extending
beyond smooth regularizers using $\Gamma$-convergence.
%draws on recent such results for smooth regularizers, 
Using this insight,
we provide a simple alternating proximal gradient algorithm %without the need for careful initialization strategies
and demonstrate the utility of learning supervised sparse coding
representations versus unsupervised sparse coding and a variety of tile-coding representations.

\section{Background}

In reinforcement learning, an agent interacts with its environment, 
receiving observations and selecting actions to maximize a scalar 
reward signal provided by the environment. 
This interaction is usually modeled by a 
Markov decision process (MDP). 
An MDP consists of $(\States, \Actions, P, R)$ where $\States$ is the set of states; 
$\Actions$ is a finite set of actions; 
$P: \States \times \Actions \times \States \rightarrow [0,1]$, the transition function, which describes the probability of reaching a state 
$s'$ from a given state and action $(s,a)$; and finally the reward function 
$R: \States \times \Actions \times \States \rightarrow \RR$, which returns a scalar value for transitioning from state-action 
$(s,a)$ to state $s'$. The state of the environment is said to be 
\textit{Markov} if $Pr(s_{t+1} | s_{t} , a_t) = Pr(s_{t+1}| s_{t} , a_t, \dots, s_{0}, a_{0})$. 

One important goal in reinforcement learning is policy evaluation: 
learning the \textit{value function} for a policy.
A value function $\Vpi: \States \rightarrow \RR$
approximates the expected return. The return $G_t$ from a state $s_t$ is the total discounted future reward, discounted by $\gamma \in [0,1)$, 
for following policy
$\pi: \States \times \Actions \rightarrow [0,1]$
\begin{align*}
G_t &= \sum_{i=0}^\infty \gamma^i R_{t+1+i}
= R_{t+1} + \gamma G_{t+1}
\end{align*}
where $\Vpi(s_t)$ is the expectation of this return from state $s_t$.
% from a given state: $\Vpi(s)$. 
This value function can also be thought of as a vector of values $\Vpi \in \RR^{\nstates}$
satisfying the Bellman equation
\begin{align}
\Vpi &= \Rpi + \gamma \Ppi \Vpi \label{eq:bellman}\\
%\end{align}
%
\text{where } \ \ 
%\begin{align*}
\Ppi (s,s') &= \sum_{a \in \Actions} \pi(s, a) P(s,a,s') \nonumber\\
\Rpi (s) &= \sum_{a \in \Actions} \pi(s, a) \sum_{s' \in \States} P(s,a,s') R(s,a,s') \nonumber
\end{align}
Given the reward function and transition probabilities,
the solution can be analytically obtained: 
$ \Vpi = (\eye - \gamma \Ppi)^{-1} \Rpi$.

In practice, however, we likely have a prohibitively large state space.
The typical strategy in this setting is to use function approximation to learn $\Vpi(s)$
from a trajectory of samples:
a sequence of
states, actions, and rewards $s_0$, $a_0$, $r_0$, $s_1$, $a_1$, $r_1$, $s_2$, $r_2$, $a_2 \ldots$,
where $s_0$ is drawn from the start-state distribution, 
$s_{t+1} \sim P(\cdot | s_t, a_t)$ and
$a_t \sim \pi(\cdot | s_t)$. 
Commonly, a linear function is assumed,
%function approximation with features describing the state:
%\begin{equation*}
$\Vpi(s) \approx \repvec(s)^\top \vparams$
%\end{equation*}
for $\vparams \in \RR^{\rdim}$ a parameter vector
and $\repvec: \States \rightarrow \RR^\rdim$ a feature function
describing states.
With this approximation, however, typically we can no
longer satisfy the Bellman equation in (\ref{eq:bellman}),
because there may not exist a $\wvec$ such that $\repall \vparams$
equals $\Rpi + \gamma \Ppi \repall \vparams$ 
for $\repall \in \RR^{\nstates \times \rdim}$.
Instead, we focus on minimizing the error 
%that obtain accurate approximations 
to the true value function. 

Reinforcement learning algorithms, such as temporal difference learning
and residual gradient,
therefore focus on finding
an approximate solution to the Bellman equation,
despite this representation issue.
The quality of the representation is critical to accurately approximating $\Vpi$
with $\repall \wvec$, but also balancing compactness of the representation
and speed of learning. Sparse coding, and sparse representations, have proven successful
in machine learning and in reinforcement learning,
particularly as fixed bases, such as tile coding, radial basis functions and other kernel representations. 
A natural goal, therefore, and the one we explore in this work,
is to investigate learning these sparse representations automatically. 

\section{Sparse Coding for Reinforcement Learning}

In this section, we formalize sparse coding for reinforcement learning as
a joint optimization over the
value function parameters and the representation.
We introduce the true objective over all states, and then move to the sampled
objective for the algorithm in the next section.
%We provide the optimization up front, and explain each of the components
%and choices below.
%
%Let $\tdobj(\repall, \vparams)$ be the chosen objective for 
%learning the value function parameters, $\vparams$.
%For example, $\tdobj(\repall, \vparams) = \MSPBE(\repall, \vparams)$ or
%$\tdobj(\repall, \vparams) = \BR(\repall, \vparams)$, described in the next section.
%In particular, for convex reformulations, we will require that $\tdobj(\repall, \vparams)$
%is convex in each parameter; in general, however, it can be any loss for which
%a gradient is computable for $\repall$.

We begin by formalizing the representation learning component.
Many unsupervised representation learning approaches consist of
factorizing input observations\footnote{This variable $\Xmat$ can also be a base set of features, 
on which the agent can improve or which the agent can sparsify.} 
%For simplicity, however, we simply assume it is input observations.} 
$\Xmat \! \in \!\RR^{\nstates \times \xdim}$ into a basis dictionary $\rparams \! \in \!\RR^{\rdim \times \xdim}$ and 
new representation $\repall \in \RR^{\nstates \times \rdim}$. The rows of $\rparams$ form a set of bases,
with columns in $\repall$ weighting amongst those bases for each observation (column) in $\Xmat$.
Though simple, this approach encompasses a broad range of models, including 
%principal components analysis, canonical correlation analysis, 
PCA, CCA,
ISOMAP, locally linear embeddings and sparse coding \citep{singh2008unified,le2016global}.
The (unsupervised) sparse coding objective is \citep{aharon2006ksvd}
%\citep{bach2008convex,mairal2009supervised,mairal2010online,zhang2011convex}
%
\begin{align*}
\min_{\repall \in \RR^{\nstates \times \rdim}, \rparams \in \RR^{\rdim \times \xdim}} \| \repall \rparams - \Xmat \|_D^2 + \regwgtb \| \rparams \|_F^2 + \regwgt \| \repall \|_{D,1}
\end{align*}
where $\| \Ymat \|_F^2 = \sum_{ij} \Ymat_{ij}^2$ is the squared Frobenius norm;
$\rparams \in  \RR^{\rdim \times \xdim}$ is a learned basis dictionary;
$\regwgtb,\regwgt > 0$ determine the magnitudes of the regularizers;
$\Dmat \in [0,1]^{\nstates \times \nstates}$ is a diagonal matrix giving a distribution over states,
corresponding to the stationary distribution of the policy $\dpi: \States \rightarrow [0,1]$; and
$||\zvec||^2_D = \zvec^\top \Dmat \zvec$ is a weighted norm.
The reconstruction error 
\begin{align*}
\| \repall \rparams - \Xmat \|_D^2 = \sum_{s \in \States} \dpi(s) \| \repall(s,:) \rparams - \Xmat(s,:) \|_2^2
\end{align*}
is weighted by the stationary distribution $\dpi$ because states are observed with frequency indicated by $\dpi$. 
The weighted $\ell_1$ 
\begin{align*}
\| \repall \|_{D,1} = \sum_{s \in \States} \dpi(s)  \sum_{j=1}^\rdim |\repall(s,j)|
\end{align*}
 promotes sparsity
on the entries of $\repall$, preferring entries in $\repall$ to be entirely pushed to zero
rather than spreading magnitude across all of $\repall$. 
The Frobenius norm regularizer on $\rparams$ ensures that $\rparams$ does not become too large.
Without this regularizer, all magnitude can be shifted to $\rparams$, producing
the same $\repall \rparams$, but pushing $\| \repall \|_{D,1}$ to zero and nullifying
the utility of its regularizer. 
Optimizing this sparse coding objective would select a sparse representation $\repvec$ for 
each observation $\xvec$ such that $\repvec \rparams$ approximately reconstructs $\xvec$.

Further, however, we would like to learn a new representation that is also 
optimized towards approximating the value function. 
Towards this aim, we need to jointly learn $\repall$ and $\wvec$, where $\repall \wvec$
provides the approximate value function. In this way, the optimization must balance between accurately
recreating $\Xmat$ and approximating the value function $\repall \wvec$.
For this, we must choose an objective for learning $\wvec$.

We consider two types of objectives:
fixed-point objectives and squared-error objectives. 
Two common fixed-point objectives are
%considered in reinforcement learning are 
the mean-squared Bellman error (MSBE), also called the Bellman residual~\citep{baird1995residual}
\begin{equation*}
\| \repall \vparams  - (\Rpi + \gamma \Ppi \repall \vparams) \|^2_D
\end{equation*}
and mean-squared projected BE (MSPBE)~\citep{sutton2009fast}
\begin{align*}
\| \repall \vparams  - \Pi (\Rpi + \gamma \Ppi \repall \vparams) \|^2_D %\label{eq:approx}
\end{align*}
where $\Dmat \in [0,1]^{\nstates \times \nstates}$ is a diagonal matrix giving a distribution over states,
corresponding to the stationary distribution of the policy;
$||\zvec||^2_D = \zvec^\top \Dmat \zvec$ is a weighted norm;
and the projection matrix for linear value functions is $\Pi = \repall (\repall^\top \Dmat \repall)^{-1} \repall^\top \Dmat$. 
The family of TD algorithms converge to the minimum of the MSPBE, 
whereas residual gradient algorithms typically use the MSBE (see \citep{sun2015online} for an overview). 
Both have useful properties \citep{scherrer2010should},
though arguably the MSPBE is more widely used. 

There are also two alternative squared-error objectives, that do not correspond to fixed-point equations: 
the mean-squared return error (MSRE) and the Bellman error (BE). 
For a trajectory of samples $\{(\xvec_i, r_{i+1}, \xvec_{i+1})\}_{i=0}^{\nsamples-1}$,
BE is defined as
\begin{align*}
\sum_{i=0}^{\nsamples-1} \| r_{i+1} + \gamma \rep_{i+1}^\top \vparams - \rep_{i}^\top \vparams \|_2^2
\end{align*}
and the MSRE as
\vspace{-0.5cm}
\begin{align*}
\sum_{i=0}^{\nsamples-1} \| g_{i+1} - \rep_{i}^\top \vparams \|_2^2
\end{align*}
where $g_{i+1} = \sum_{j=i}^{\nsamples-1} \gamma^{j-i} r_{j+1}$ is a sample return. 
In expectation, these objectives are, respectively
 \small
\begin{align*}
&\!\!\!\sum_{s \in \States} \! \dpi\!(s) \mathbb{E} \!\!\left[ \!\!\left( r(S_t,A_t,S_{t+1}) \!+ \!\gamma \rep(S_{t+1})^\top \vparams - \rep(S_t)^\top \vparams \right)^2 \!\!| S_t \!\!= \!\!s \right]\\
&\!\!\!\sum_{s \in \States} \!\dpi\!(s) \mathbb{E} \!\left[\!\!\left(\sum_{i=0}^\infty \gamma^i r(S_{t+i},A_{t+i},S_{t+1+i}) - \rep(s)^\top \vparams \right)^2 \!\!\!\!| S_t \!\!= \!\!s \right]
\end{align*}
 \normalsize
 where the expectation is w.r.t. the transition probabilities and taking actions according to policy $\pi$. 
 
These differ from the fixed-point objectives because of the placement of the expectation.
To see why, consider the MSBE and BE. The expected value of the BE 
is the expected squared error between the prediction from this state and
the reward plus the value from a possible next state.
The MSBE, on the other hand, 
is the squared error between the prediction from this state and
the expected reward plus the expected value for the next state. 
%The BE has been used for an online algorithm \citep{sun2015online}
Though the MSPBE and MSBE constitute the most common objectives chosen for reinforcement learning,
these squared-error objectives have also been shown to be useful particularly for learning online \citep{sun2015online}.

For sparse coding, 
%and factorized representations in general, 
however, 
the MSPBE is not a suitable choice---compared to the MSBE, BE and MSRE---for two reasons.
First, the MSBE, BE and MSRE are all convex in $\repall$, whereas the MSPBE is not. 
Second, %the MSPBE is not suitable for this joint optimization:
because of the projection onto the space spanned by the features,
the MSPBE can be solved with zero error for any features $\repall$.
%The closed form solution to the MSPBE is
%\begin{equation*}
%\vparams = \left(\repall ^\top D \repall\right)^{-1} \repall^\top D\left(\Rpi + \gamma \Ppi \repall \vparams\right)
%\implies
%\repall \vparams = \Pi \bellman{\repall \vparams}
%%\repall \vparams = \Pi \left(\Rpi + \gamma \Ppi \repall \vparams\right) = \Pi \bellman{\repall \vparams}
%\end{equation*}
%Therefore, the MSPBE would not inform the choice of $\repall$,
%and so there would not be a preference for $\repall$ that improve value function estimation,
%over simply approximating observations. 
Therefore, because it does not inform the choice of $\repall$, the MSPBE produces a two stage approach\footnote{This problem seems to have been overlooked in
two approaches for basis adaptation based on the MSPBE:
adaptive bases algorithm for the projected Bellman error (ABPBE) \cite[Algorithm 9]{castro2010adaptive}
and 
mirror descent Q($\lambda$) with basis adaptation \citep{mahadevan2013basis}.
For example, for ABPBE, it is not immediately obvious this would be a problem,
because a stochastic approximation approach is taken. However, if written as a minimization
over the basis parameters and the weights, one would obtain a minimum error solution (i.e., error zero)
immediately for any basis parameters. The basis parameters are considered to change on a slow timescale,
and the weights on a fast timescale, which is a reflection of this type of separate minimization.
\citet{menachie2005basis} avoided this problem by explicitly using a two-stage approach, using MSPBE approaches for learning
the parameters and using other score function, such as the squared Bellman error, to
update the bases. This basis learning approach, however, is unsupervised.\\
Representation learning strategies for the MSPBE have been developed, by using local projections \citep{yu2009basis,maei2009convergent}. These strategies, however, do not incorporate sparse coding. 
%are more akin to feed-forward neural networks, result in difficult
%nonconvex optimization problems and do not explicitly take a sparse coding approach.
}, where 
features are learned in a completely unsupervised way and prediction
performance does not influence $\repall$.

The final objective for loss $\textrm{L}(\repall, \vparams)$ set to either MSBE, BE or MSRE is
\begin{align}
\min_{\vparams \in \RR^{\rdim},\repall \in \RR^{\nstates \times \rdim}, \rparams \in \RR^{\rdim \times \xdim}}
&\textrm{L}(\repall, \vparams) + \| \repall \rparams - \Xmat \|_D^2 \label{eq_scope}\\
%\tfrac{1}{\nstates} \| \repall \vparams  - (\Rpi + \gamma \Ppi \repall \vparams) \|^2_D \label{eq_scope}\\
%&+ \tfrac{\fwgt}{\nstates} \| \repall \rparams - \Xmat \|_F^2 
 &+ \regwgtw \| \vparams \|_2^2 + \regwgtb \| \rparams \|_F^2 + \regwgt\| \repall \|_{D,1} \nonumber
\end{align}

\section{Algorithm for Sparse Coding}

We now derive the algorithm for sparse coding for policy evaluation: \scalg. 
%In this section, we derive an algorithm for 
We generically consider either the BE or MSRE. 
For a trajectory of samples $\{(\xvec_i, r_{i+1}, \xvec_{i+1})\}_{i=0}^{\nsamples-1}$, 
%%
%\begin{align*}
%\tdobj(\repall, \vparams) = 
%\frac{1}{\nsamples} \sum_{i=0}^{\nsamples-1} \| r_{i+1} + \gamma \rep_{i+1}^\top \vparams - \rep_{i}^\top \vparams \|_2^2
%\end{align*}
%%
%A typical setting in both cases will be $ L(\repall \rparams, \Xmat) = \tfrac{1}{\nsamples} \sum_{i=0}^{\nsamples} \| \rep_i \rparams - \xvec_i\|_2^2$, though any convex function could be used here.
%For the regularizer, we have two settings to consider: differentiable regularizers on $\repall$ or non-differentiable regularizers.
%This distinction is important in terms of using gradient descent. With a differentiable regularizer,
%the results below are simplified to using an alternating minimization. 
%To enable a non-differentiable regularizer, such as the $\ell_1$, however, we will pursue a slightly different strategy.
%To make the approach more concrete, 
the objective is
\begin{align}
&\min_{\vparams \in \RR^{\rdim},\repall \in \RR^{\nsamples+1 \times \rdim}, \rparams \in \RR^{\rdim \times \xdim}} 
\frac{1}{\nsamples}\sum_{i=0}^{\nsamples-1} ( y_i + \bar{\gamma} \rep_{i+1}^\top \vparams - \rep_{i}^\top \vparams )^2 \label{eq_sampleloss}\\
&\!\!+ \frac{\fwgt}{\nsamples} \!\sum_{i=0}^{\nsamples} \!\| \rep_i \rparams - \xvec_i\|_2^2 
 + \! \regwgtb \| \rparams \|_F^2
 + \! \regwgtw \| \vparams \|_2^2 
\! + \! \frac{\regwgt}{\nsamples}  \sum_{i=0}^{\nsamples} \|\repvec_i \|^p_{1}
.\nonumber
\end{align}
for BE, $y_i = r_{i+1}$ and $\bar{\gamma} = \gamma$
and for MSRE,
$y_i = \sum_{j=i}^{\nsamples} \gamma^{j-i} r_{j+1}$ and $\bar{\gamma} = 0$. 
We consider two possible powers for the $\ell_1$ norm $p = 1$ or $2$,
where the theory relies on using $p=2$, but in practice we find they perform equivalently
and $p=1$ provides a slightly simpler optimization.
%At the end of the section, we discuss how this can be extended to other regularizers. 
The loss is averaged by $\nsamples$, to obtain a sample average, which in the limit converges
to the expected value under $\dpi$. 
This averaged loss is also more scale-invariant---in terms of the numbers of samples---to the choice of regularization parameters.
%This scale-invariant choice is key for accurate out-of-sample prediction with \scalg.
%The algorithm is not only limited to MSBE loss, 
%but also able to deal with MSRE loss, by setting $\gamma = 0$ and replacing $r_{i+1}$ with $\sum_{j=0}^{i+1} r_j$. \\

\scalg\ consists of alternating amongst these three variables, $\rparams, \vparams$ and $\repall$,
with a proximal gradient update for the non-differentiable $\ell_1$ norm.
The loss in terms of $\rparams$ and $\vparams$ is differentiable; to solve for $\rparams$ (or $\vparams$) with the other variables fixed,
we can simply used gradient descent. 
To solve for $\repall$ with the $\rparams$ and $\vparams$ fixed, however, we cannot use
a standard gradient descent update because the $\ell_1$ regularizer is non-differentiable. 
The proximal update consists of stepping in the direction of the gradient for the smooth component
of the objective---which is differentiable---and then projecting back to a sparse solution
using the proximal operator: a soft thresholding operator. 
The convergence of this alternating minimization follows from results on block coordinate descent
for non-smooth regularizers \citep{xu2013ablock}.

%This full alternating minimization algorithm is summarized in Algorithm \ref{Alg:1}. 
To apply the standard proximal operator for the $\ell_1$ regularizer,
we need to compute an upper bound on the Lipschitz constant for this objective. 
The upper bound is $2(1+\bar{\gamma}^2)\|\vparams\|_2^2 + 2\|\rparams\|_{sp}^2$,
computed by finding the maximum singular value
of the Hessian of the objective w.r.t. $\rep_i$ for each $i$. We will provide
additional details for this calculation, and implementation details, in a supplement. 

% Do we need this?
%Once we learn $\triplet$ on the observed data, we need to be able to extract the representation
%for a new sample. 
%To obtain the representation for a set of new samples $\Xnew \in \RR^{\nsamplesnew \times \xdim}$, we simply have to solve the optimization over $\repnew$ with the
%dictionary fixed. 
%%%
%%\begin{align*}
%%\min_{\repall \in \RR^{\nsamplesnew \times \rdim}} & \frac{1}{\nsamplesnew}\sum_{i=0}^{\nsamplesnew-1} (y_{i} + \bar{\gamma} \rep_{i+1}^\top \vparams - \rep_{i}^\top \vparams )^2 \\
%%&+ \frac{\fwgt}{\nsamplesnew} \sum_{i=0}^{\nsamples} \| \rep_i \rparams - \xvec_i\|_2^2 + \frac{\regwgt}{\nsamplesnew} \sum_{i=0}^{\nsamplesnew-1}  \|\repvec_i \|_{1}
%%\end{align*}
%%%
%%%The optimization details for out-of-sample prediction is included in Algorithm \ref{alg_out} in the appendix. 

%\subsection{Convergence to globally optimal solutions}
\subsection{Local Minima Are Global Minima}

In this section, we show that despite nonconvexity,
the objective for \scalg\ has the nice property that all local 
minima are in fact global minima.
Consequently, though there may be many different local minima,
they are in fact equivalent in terms of the objective. 
This result justifies a simple alternating minimization scheme, where convergence 
to local minima ensures an optimal solution is obtained. 
%%This result is of independent interest outside of reinforcement learning,
%%as previous similar results have required twice-differentiable regularizers \citep{le2016global}
%%or
% FR-BRM converges to a global solution, despite the fact
%that an alternating minimization is used on a non-convex objective. 
%Using recent results in \citep{le2016global}, alternating minimization on this objective is highly likely
%to produce global solutions. We include those results by proving that the above objective is
%an induced regularized factor model objective, enabling \cite[Theorem 1-3]{le2016global} to apply here.
%We extend these previous results to also include non-smooth regularizers, including $\ell_1$,
%which is what we explore in the experiments. Further, we prove the the alternating minimization
%converges to a stationary point. Combined these results indicate that FR-BRM converges to a globally optimal solution. 

%\begin{lemma}
%The FR-BRM objective is an induced regularized factor model objective. 
%\end{lemma}
%\begin{proof}
%To be an induced RFM, we simply need
%\begin{enumerate}
%\item The loss $L(\repall \jparams) = \| (\eye - \gamma \Ppi) \repall \jparams - R \|_D^2$ must be convex in $\joint = \repall \jparams$.  
%\item The regularizers must be convex.
%\end{enumerate}
%For convexity of the loss, the composition of a convex function and an affine function is convex; multiplication of $\joint$ by $(\eye - \gamma \Ppi)$ 
%therefore maintains convexity. 
%\end{proof}

\newcommand{\regth}{f}
\newcommand{\lossth}{L}
\newcommand{\paramset}{\Theta}
\newcommand{\params}{\thetavec}

We need the following technical assumption. It is guaranteed to be true for a sufficiently large $\rdim \le \nsamples$ (see 
\cite{haeffele2015global,le2016global}).
%removed bach citation for space: \cite{bach2008convex,haeffele2015global,le2016global}).
%

\begin{figure*}[ht]
\vspace{-0.5cm}
\begin{subfigure}[b]{0.345\textwidth}
        \includegraphics[width=\textwidth]{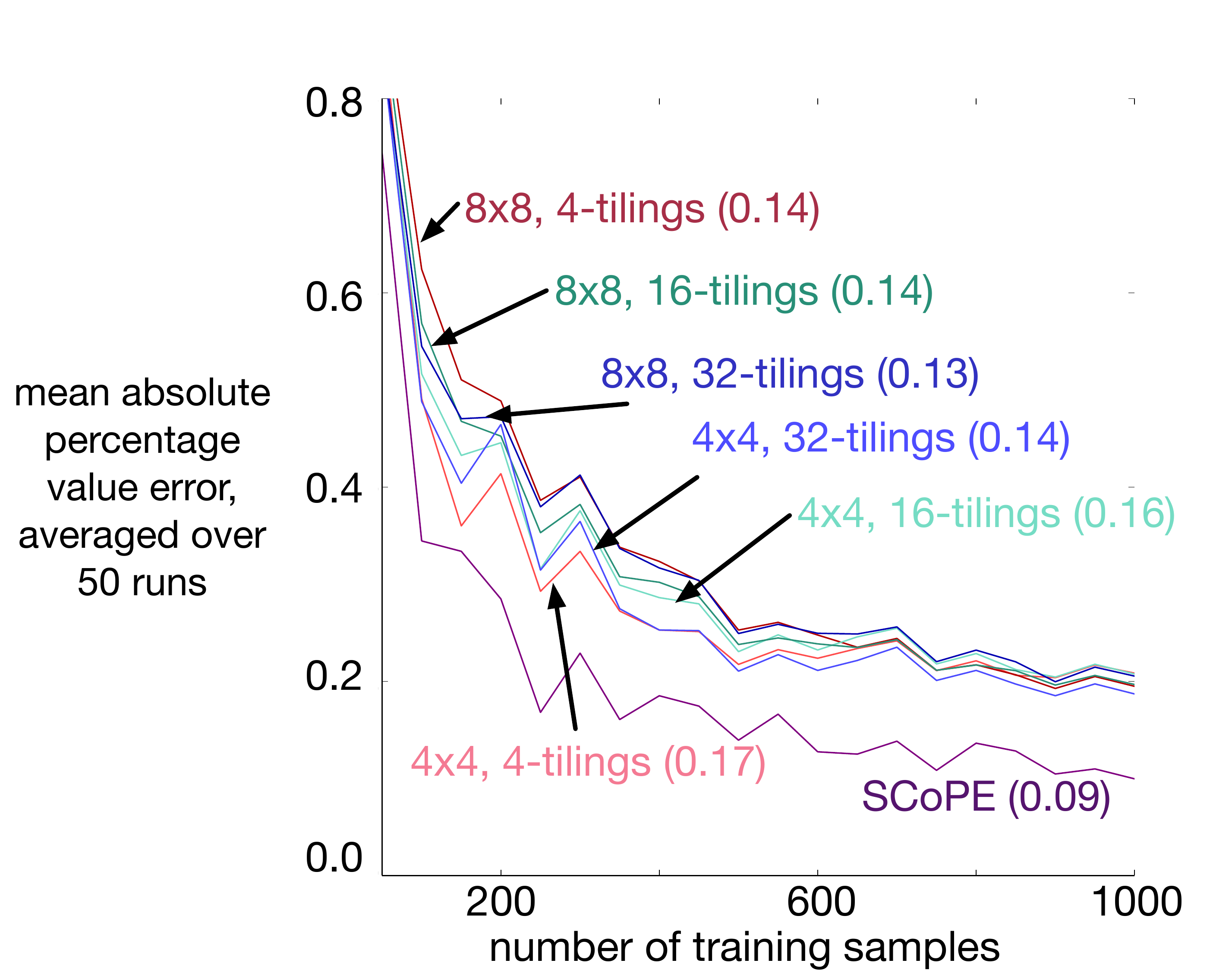}
        \caption{Mountain Car}
        \label{fig:MC}
    \end{subfigure}
\begin{subfigure}[b]{0.27\textwidth}
        \includegraphics[width=\textwidth]{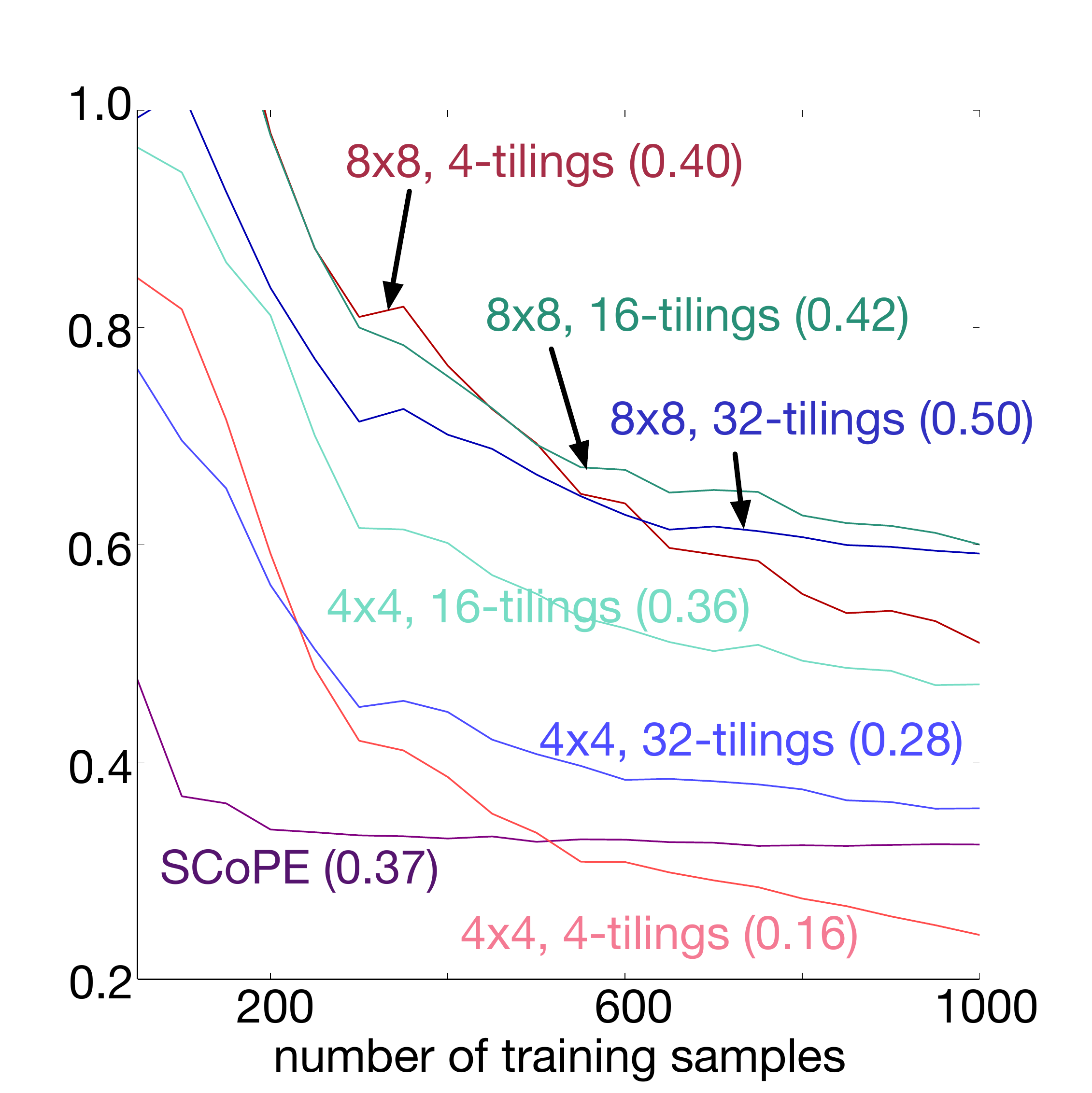}
        \caption{Puddle World }
        \label{fig:PW}
    \end{subfigure}
\begin{subfigure}[b]{0.275\textwidth}
        \includegraphics[width=\textwidth]{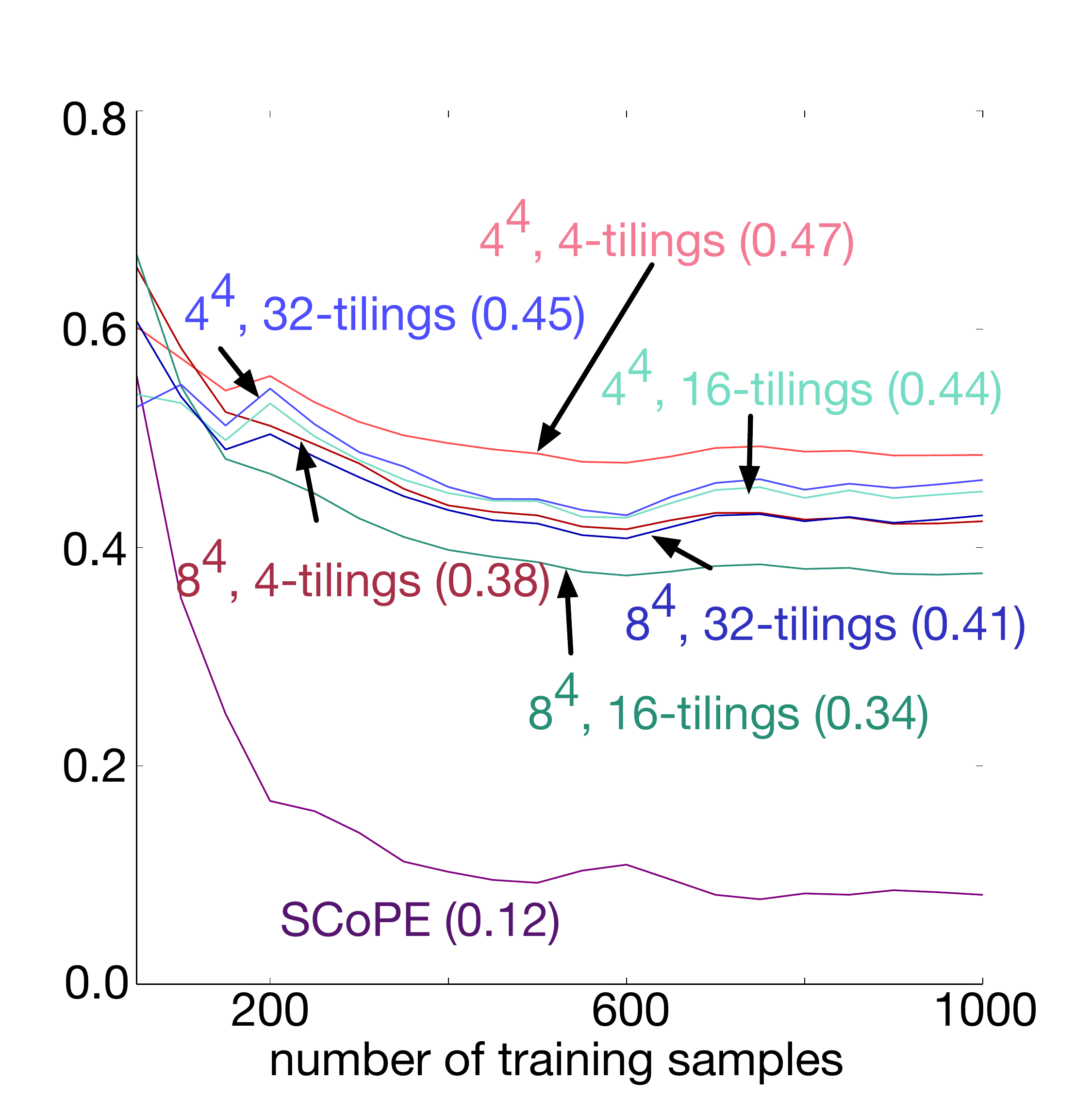}
        \caption{Acrobot }
        \label{fig:AC}
    \end{subfigure}    
\caption{Learning curves for \scalg\ versus a variety of tile coding representations in three domains. 
The graphs depict early learning; the numbers
in brackets correspond to final error, after 5000 samples. Because errors are sampled every 50 samples, and because
MSRE is used for optimization, the lines have an irregular pattern. The differences are nonetheless statistically significant, 
with an average over 50 runs, and so the small standard error bars are omitted. \scalg\ outperforms the best of the TC representations in \MC\ and \AC\, using a more compact sparse representation; in \PW, it performs more poorly, which we discuss further in the text. The larger TC representations likely perform poorly 
due to hashing.}
\label{fig:LC}
\vspace{-0.2cm}
\end{figure*}

\vspace{0.05cm}
\noindent
\textbf{Assumption 1}
\textit{For the given $\rdim \ge \xdim$, the following function is convex in $\mathbf{Z} \in \RR^{\nsamples+1 \times \xdim+1}$ 
\vspace{-0.5cm}
\begin{equation*}
\min_{\repall, \rparams, \vparams, \mathbf{Z} = \repall \inlinervec{\rparams}{\vparams}} \regwgtb \| \rparams \|_F^2 + \regwgtw \| \vparams \|_2^2 + \frac{\regwgt}{\nsamples}  \sum_{i=0}^{\nsamples} \|\repvec_i \|^2_{1}
\end{equation*}
}
\begin{theorem}[Landscape of the \scalg\ objective]\label{thm_global}
For the objective in equation \eqref{eq_sampleloss} with $p = 2$, 

\noindent
1. under Assumption 1, all full-rank local minima are global minima; and

\noindent
2. if a local minimum $(\repall, \rparams, \vparams)$ has $\repall_{:i} = \zerovec$ (i.e., a zero column) and $\vparams_i = 0$, $\rparams_{i:} = \zerovec$ (i.e., a zero row) for some $1 \le i \le \rdim$, then it is a global minimum. 
\end{theorem}
\begin{proof}
\textbf{For the first statement}, we construct a limit of twice-differentiable functions $f_n$ that $\Gamma$-converge to
the \scalg\ objective $f$. With this, we can then show that all minimizers of the sequence
converge to minimizers of $f$, and vice-versa \citep{braides2013local}. % Better citation \citep{braides2011gamma}.
Because all local minimizers of the twice-differentiable functions $f_n$
are global minimizers from \cite[Theorem 10]{le2016global}, we can conclude that
all corresponding minimizers of $f$ are global minimizers. 

We use the pseudo-Huber loss \citep{fountoulakis2013asecond},
which is twice-differentiable approximation to the absolute value:
%
%\begin{align*}
$|x|_{\mu} = \sqrt{\mu^2 + x^2} - \mu$.
%\end{align*}
%
Let $\params = (\repall, \rparams, \vparams)$. The sequence of functions $f_n$ are defined with
$\mu_n = 1/n$, as
\begin{align*}
f_n(\params) = L(\params) + \frac{\regwgt}{\nsamples}\sum_{ij} \sqrt{\mu_n^2 + \repall_{ij}^2} - \mu_n
\end{align*}
where $L(\params)$ equals the equation in \eqref{eq_sampleloss}, but without the $\ell_1^2$ regularizer on $\repall$. 
% the loss and regularizers for $\rparams,\vparams$

%\vspace{0.05cm}
\noindent
\textbf{Part 1:} All local minima of $f_n$ for all $n$ are global minima.
To show this, we show each $f_n$ satisfies the conditions of \cite[Theorem 10 and Proposition 11]{le2016global}.
% and that there are no degenerate saddlepoints.
%
%\begin{enumerate}[leftmargin=3ex,topsep=1pt,itemsep=1ex,partopsep=1ex,parsep=1ex]
%\item 

\noindent
\textbf{Part 1.1} 
%To show $L(\params)$ is convex,
%and can be written as a function of $\repall \inlinevec{\rparams}{\vparams}$. 
We can rewrite the loss in terms
of $\repall \inlinervec{\rparams}{\vparams}$
\vspace{-0.25cm}
\begin{align}
&\!\!\!\!\frac{1}{\nsamples}\sum_{i=0}^{\nsamples-1} \| y_{i} + \bar{\gamma} \rep_{i+1}^\top \vparams - \rep_{i}^\top \vparams \|_2^2  
+ \frac{\fwgt}{\nsamples} \sum_{i=0}^{\nsamples} \| \rep_i \rparams - \xvec_i\|_2^2 \nonumber\\
&\!\!\!\!=
\frac{1}{\nsamples} \left\| \Xmat -  \repall \rparams \right\|_2^2
\!+\! \frac{1}{\nsamples} \left\| \mathbf{y} -  \left(\eye_{0:\nsamples-1} \!-\! \bar{\gamma} \eye_{1:\nsamples} \right)\repall \vparams \right\|_2^2
%\frac{1}{\nsamples} \left\| \inlinevec{\Xmat}{\mathbf{y}} -  \inlinevec{\eye}{\eye_{0:\nsamples-1} - \bar{\gamma} \eye_{1:\nsamples}}\repall \inlinevec{\rparams}{\vparams}  \right\|_2^2
 \label{eq_joint}
\end{align}
where $\eye_{1:\nsamples} \in \RR^{\nsamples \times \nsamples+1}$ a diagonal matrix of all ones with the first diagonal entry set to zero, and $\eye_{0:\nsamples-1}$ with the last diagonal entry set to zero. 
%$\frac{1}{\nsamples}\sum_{i=0}^{\nsamples-1} \| r_{i+1} + \gamma \rep_{i+1}^\top \vparams - \rep_{i}^\top \vparams \|_2^2 + 
%+ \frac{\fwgt}{\nsamples} \sum_{i=0}^{\nsamples} \| \rep_i \rparams - \xvec_i\|_2^2$
This loss
is convex in the joint variable $\repall \inlinervec{\rparams}{\vparams}$ because equation \eqref{eq_joint} is
the composition of a convex function (squared norm) and an affine function (multiplication by $\bar{\gamma} \eye_{1:\nsamples} - \eye_{0:\nsamples-1}$ and addition of $\mathbf{y}$).

\noindent
\textbf{Part 1.2}
The regularizer on  $\inlinervec{\rparams}{\vparams}$ must be a weighted Frobenius norm, with weightings on each column;
here, we have weighting using regularization parameters $\regwgtb$ for the first $\xdim$ columns (corresponding to $\rparams$)
and regularization parameter $\regwgtw$ for the last column (corresponding to $\vparams$).

\noindent
\textbf{Part 1.3}
The inner dimension $\rdim > \xdim$, which is true by assumption
and the common setting for sparse coding.

\noindent
\textbf{Part 1.4}
The pseudo-Huber loss, on the columns of $\repall$, is convex, centered and twice-differentiable.
%\end{enumerate}

\vspace{0.2cm}
\noindent
\textbf{Part 2:}
The sequence $f_n$ converges uniformly to $f$. To see why, recall the definition
of uniform convergence. A sequence of functions $\{ f_n \}$ is uniformly convergent with limit $f$
if for every $\epsilon > 0$, there exists $N \in \NN$ such that for all $\params \in \paramset$ all $n \ge N$,
$|f_n(\params) - f(\params)| < \epsilon$. 
Further recall that for any complete metric space, if $f_n$ is uniformly Cauchy, then it is uniformly convergent.
The sequence is uniformly Cauchy if for all $n,m \ge N$, $|f_n(\params) - f_m(\params)| < \epsilon$.
% MARTHA's proof
Take any $\epsilon > 0$ and let $N = \lceil \tfrac{4 \rdim(\nsamples+1)\regwgt}{\nsamples\epsilon} \rceil$. Then 
\begin{align*}
&|f_n(\params) - f_m(\params)| \\
&=  \frac{\regwgt}{\nsamples}\left|\!\!\left( \!\!\sum_{ij} \sqrt{\mu_n^2 + \repall_{ij}^2} - \mu_n\right) \!\! -\!\! \left(\!\!\sum_{ij} \sqrt{\mu_m^2 + \repall_{ij}^2} - \mu_m \right)\!\! \right|\\
%&\le \frac{\regwgt}{\nsamples}\sum_{ij} \left| \sqrt{\mu_n^2 + \repall_{ij}^2} - \mu_n - \sqrt{\mu_m^2 + \repall_{ij}^2} + \mu_m \right|\\
&\le \frac{\regwgt}{\nsamples}\sum_{ij} \left(\left| \sqrt{\mu_n^2 + \repall_{ij}^2} - \sqrt{\mu_m^2 + \repall_{ij}^2} \right| + \left| \mu_n -\mu_m \right| \right)
\end{align*}
The upper bound of the first component is maximized when $\repall_{ij} = 0$, and so we get
\begin{align*}
\!\! |f_n(\params) \!-\! f_m(\params)| &\! \le  \! \tfrac{2\rdim(\nsamples+1)\regwgt}{\nsamples} \left| \mu_n \!-\! \mu_m \right|
\! \le \! \tfrac{2\rdim(\nsamples+1)\regwgt}{\nsamples} \left| \tfrac{1}{n} \! -\! \tfrac{1}{m} \right|\\
&\! \le \! \tfrac{4\rdim(\nsamples+1)\regwgt}{\nsamples N} \le \epsilon
.
\end{align*}

% LEI's proof
%Take any $\epsilon > 0$ and let $N = \lceil \frac{4 \sum_{ij}\regwgt}{T\epsilon} \rceil$. Then 
%%
%\begin{align*}
%&|f_n(\params) - f_m(\params)| \\
%&=\frac{\regwgt}{\nsamples}\left|\sum_{ij} \left(\sqrt{\mu_n^2 + \repall_{ij}^2} - \mu_n \right) - \sum_{ij} \left(\sqrt{\mu_m^2 + \repall_{ij}^2} - \mu_m\right)\right|\\
%&\le \frac{\regwgt}{\nsamples}\sum_{ij} \left(\left| \sqrt{\mu_n^2 + \repall_{ij}^2} - \sqrt{\mu_m^2 + \repall_{ij}^2} \right|+ \left|\mu_m - \mu_n\right|\right)\\
%&=\frac{\regwgt}{\nsamples}\sum_{ij} \left(\left| \frac{\mu_n^2-\mu_m^2}{\sqrt{\mu_n^2 + \repall_{ij}^2} + \sqrt{\mu_m^2 + \repall_{ij}^2}}\right|+ \left|\mu_m - \mu_n\right|\right)\\
%&\le \frac{\regwgt}{\nsamples}\sum_{ij} \left(\left| \frac{\mu_n^2-\mu_m^2}{\sqrt{\mu_n^2} + \sqrt{\mu_m^2}}\right|+ \left|\mu_m - \mu_n\right|\right)\\
%&\le \frac{2\regwgt}{\nsamples}\sum_{ij}\left|\mu_n - \mu_m\right|\\
%&= \frac{2\regwgt}{\nsamples}\sum_{ij}\left|\frac{1}{n} - \frac{1}{m}\right|\\
%&\le \frac{2\regwgt}{\nsamples}\sum_{ij}\left(\left|\frac{1}{n}\right| + \left|\frac{1}{m}\right|\right)\le \epsilon\\
%\end{align*}
%%
%
%%This upper bound is maximized when $\repall_{ij} = 0$, and so we get
%%
%%\begin{align*}
%%|f_n(\params) - f(\params)| &\le \regwgt \left| \mu_n - \mu_m \right|\\
%%&\le \regwgt \left| \frac{1}{n} - \frac{1}{m} \right| \\
%%&\le \frac{2\regwgt}{N} \le \epsilon
%%.
%%\end{align*}

%

\noindent
\textbf{Part 3:} Asymptotic equivalence of minimizers of $f_n$ and $f$. 
Because $f$ is continuous, and so lower semi-continuous, and $f_n$ uniformly converges to $f$,
we know that $f_n$ $\Gamma$-converges to $f$: $f_n \gammarrow f$ \cite{braides2013local}. % Old citation: \cite[Pg. 12]{braides2006ahandbook}.

By the fundamental theorem of $\Gamma$-convergence,
if the $\{ f_n \}$ is an equi-coercive family of functions, then
the minimizers of $f_n$ converge to minimizers of $f$. 
A sequence of functions $\{ f_n \}$ is equi-coercive iff
there exists a lower semi-continuous coercive function $\psi: \paramset \rightarrow \RR \cup \{ -\infty, \infty \}$ such that
$f_n \ge \psi$ on $\paramset$ for every $n \in \NN$ \cite[Proposition 7.7]{maso2012anintroduction}. A function is coercive if $\psi(\theta) \rightarrow \infty$
as $\| \theta \| \rightarrow \infty$. For $\psi(\theta) = L(\params)$, it is clear that $\psi$ is coercive, as well as lower semi-continuous (since it is continuous). Further, $f_n(\params) \ge L(\params) = \psi(\params)$, because the regularizer
on $\repall$ is non-negative. Therefore, the family $\{ f_n \}$ is equi-coercive, and so the minimizers of $f_n$ converge to minimizers of $f$.

For the other direction, if a local minimum $\params$ of $f$ is an isolated local minimum, then there exists a sequence $\params_n \rightarrow \params$ with $\params_n$ a local minimizer of $f_n$ for $\mu_n$ sufficiently small \cite[Theorem 5.1]{braides2013local}.
Because we have Frobenius norm regularizers on $\rparams, \vparams$, which are strongly convex,
the objective is strictly convex with respect to $\rparams, \vparams$. Further, because $\Xmat$ is full rank,
 $\| \repall \rparams - \Xmat \|_F^2 $ is a strictly convex function with respect to $\repall$.
 Therefore, locally the objective is strictly convex with respect to $\params$. We therefore know that local minima of $f$ are isolated, and so there exists an $N$ such that for all $n > N$, $\params_n$ are local minimizers of $f_n$. Since these local minimizers are global minimizers, 
and they converge to $\params$, this means $\params$ is a global minimum of $f$.  
% MARTHAC: Old approach using uniform convergence
%From \cite[Theorem 3]{le2016global}, we know that $\lossth_\sigma$ 
%results in global solutions, because $\rdim > \xdim$, the regularizer is twice-differentiable and the $\ell_2$
%regularizer is used for both $\rparams$ and $\wvec$. Further, there are no degenerate saddlepoints
%for this objective. We make a limit argument, combined with Lemma \ref{lem_smooth}.
%
%
%MARTHAC: % MARTHAC: maybe it is still possible to introduce an auxiliary variable U
% that can be regularized with non-smooth losses, and have a smooth regularizer on Phi
% but the proximal gradient ends up adding the loss
% || (Phi - stepsize B^T (B Phi - X) - U||_2^2 + ||U||_1
% which multiples (I - B^T B) Phi, instead of only BPhi
%An alternative is to directly use the proximal reformulation. In this case, the optimization with smooth and
%non-smooth part for $\repall$ has an auxiliary variable added and the inner optimization over parameters
%that interact once more becomes smooth. The only difference is that we also need to modify the proof
%to allow non-centered regularizers, which we have already noticed that we can (i.e., add $\| \repall - \Umat \|_F^2$).
%For non-smooth regularizer $R$, the loss is then
%%
%\begin{align*}
%\min_{\Umat} \min_{\repall, \wvec, \rparams} L(\wvec, \rparams, \repall) + R(\Umat) + \frac{1}{2} \| (\repall - gradient)- \Umat \|_F^2
%\end{align*}
%%
%%
%We know that we can get a global solution for the inner minimization, and that
%the optimization over $\Umat$ is not coupled with the other three variables.
%But, 

\textbf{For the second statement}, we use \citep[Theorem 15]{haeffele2015global}. Because we already showed above that our loss can be cast as factorization, it is clear our loss and regularizers are positively homogenous, of order 2. A minimum is guaranteed to exist for our objective, because the loss function is continuous, bounded below (by zero) and goes to infinity as the parameters go to $\pm \infty$.
\end{proof}

%\begin{figure*}
%\includegraphics[scale=0.3]{figures/MC-calcMSVE.pdf}
%\centering
%\caption{Learning curves for Mountain Car.}
%\label{fig:MC}
%\end{figure*}
%
%\begin{figure*}
%\includegraphics[scale=0.3]{figures/PW-calcMSVE.pdf}
%\centering
%\caption{Learning curves for Puddle World.}
%\label{fig:PW}
%
%%[width=16cm,height=8cm]
%%
%%\includegraphics[width=16cm,height=8cm]{figures/PW-EN100vs500.png}
%\includegraphics[scale=0.3]{figures/AC-calcMSVE.pdf}
%\centering
%\caption{Learning curves for Acrobot.}
%\label{fig:PW-EN100vs500}
%\end{figure*}

%We propose therefore to use alternating minimization for \scalg. 
%The convergence of alternating minimization follows from results on block coordinate descent
%for non-smooth regularizers \citep{xu2013ablock}. If this convergence is to a local minima,
%we know that we have obtained an optimal solution. 

%\begin{theorem}[Block coordinate descent]
%%If $\rdim > \xdim$, 
%Algorithm \ref{alg_frbrm} converges to a stationary point of the \scalg\ objective. 
%\end{theorem}
%\begin{proof}
%From Theorem \ref{thm_global}, we know that if we converge to a local minimum, then we converge to a global minimum.
%We show that the block coordinate descent converges to a stationary point, and TODO further show that this stationary point
%is a local minimum (this last part may not yet be possible).
%
%TODO: block coordinate descent
%\end{proof}

\vspace{-0.3cm}
%%%%%%%%%%%%%%%%%%%%%%%%%%%%%%%%%%%%%%%%%%%%%%%
\section{Experimental Results}

We aim to address the question: can we learn useful representations using \scalg? We therefore tackle the setting where the representation is first learned, and then used, to avoid conflating incremental estimation and the utility of the representation. We particularly aim to evaluate estimation accuracy, as well as qualitatively understanding the types of sparse representations learned by \scalg. 
%the utility of the representation by measuring: 1) convergence rate of the algorithms (speed of learning), 
%%2) efficiency of the representation for out-of-sample prediction;
%and 2) final performance (final error), by evaluating the percentage error.

%\iffalse
%\begin{figure*}[t]
%\minipage{0.33\textwidth}
%  \includegraphics[width=\linewidth]{figures/MC/4-MSBE.png}\\
%%	\center{Beta sensitivity.}
%  \label{fig:MC-MSBE-4}
%\endminipage\hfill
%\minipage{0.32\textwidth}
%  \includegraphics[width=\linewidth]{figures/MC/16-MSBE.png}
% % \center{Eta sensitivity.}
%  \label{fig:MC-MSBE-16}
%\endminipage\hfill
%\minipage{0.32\textwidth}%
%  \includegraphics[width=\linewidth]{figures/MC/32-MSBE.png}
%%  \center{Lambda sensitivity.}
%  \label{fig:MC-MSBE-32}
%\endminipage
%\caption{MSBE performance in Mountain Car - sparsities 4, 16, and 32}
%\label{fig:MC-MSBE}
%\end{figure*}
%\fi

\vspace{0.1cm}
\noindent
\textbf{Domains.} % \& configurations}
We conducted experiments in three benchmark RL domains - Mountain Car, Puddle World and Acrobot \citep{sutton1996generalization}. 
All domains are episodic, with discount set to 1
until termination. 
%For both, \MC\ and \PW,we use a variety of tile-coding (TC) settings that increase the number of features, by making the grid-size more fine grained and by increasing the number of tilings. We explore 4x4x4, 4x4x16, 4x4x32, 16x16x4, 16x16x16, 16x16x32, where the total number of features respectively  are 32, 256, 512, 1024, 4096, 8192. The configurations are of the form NxNxD, where NxN denotes the granularity of the tile coder, and D denotes the number of tilings, i.e., the sparsity. The features in tile-coding algorithms are hashed to a linear space of 1024 dimensions. We ensure that the number of dimensions for TC is a power of 2, since it utilizes hashing algorithms optimized for such a configuration. The regularization weights for various sparsities are swept in the range of [1e-5  - 1e-1].
%
The data in \MC\ is generated using the standard energy-pumping policy policy with 10\% randomness. 
The data in \PW\ is generated by a policy that chooses to go North with 50\% probability, and East with 50\% probability on each step, with the starting position in the lower-left corner of the grid, and the goal in the top-right corner. 
The data in \AC\ is generated by a near-optimal policy.
% learned using a True Online TD($\lambda$) agent \citep{van2016true}, with $\lambda = 0.9$. The behavior policy here is learned in 1 run, and 3000 episodes, following which it is utilized to generate the batch data for our experiments.

%As we utilize MSRE for learning, the return estimates for each state, in both the domains, is generated by accumulating future rewards obtained from this state. Since we are utilizing batch training, it is straightforward to compute this.
%It would be interesting to learn these weights by utilizing the batch in future work. (?)

\vspace{0.1cm}
\noindent
\textbf{Evaluation.}
We measure value function estimation accuracy using
 mean absolute percentage value error (MAPVE), with rollouts to compute the true value estimates. 
 $\text{MAPVE} = \frac{1}{t_{test}} \sum_{s\in X_{test}}\frac{|\hat{V}(s) - V^*(s)|}{|V^*(s)|}$, where 
$X_{test}$ is the set of test states,
$t_{test}=5000$ is the number of samples in the test set,
 $\hat{V}(s)$ is the estimated value of state $s$
 and $V^*(s)$ is the true value of state $s$ computed using extensive rollouts.
 Errors are averaged over 50 runs. 
%
%Since this is a prediction task, we evaluate performance by using a percentage mean-square value error (MSVE\%).
%The value error is the absolute difference between the predicted value and the true value from a state $s$: $|\hat{V}(s) - V^*(s)|$.
%Given a set of states $X_{test}$, the MSVE\% is defined as
%\begin{align*}
%\text{MSVE\%} &= \frac{100}{t_{test}} \sum_{s\in X_{test}}\frac{|\hat{V}(s) - V^*(s)|}{|V^*(s)|} \\
%&\text{where,}\\
%&\hat{V}(s) - \text{estimated value of state } s\\
%&V^*(s) - \text{true value of state } s\\
%&X_{test} - \text{the set of test states}\\
%&t_{test} - \text{the number of samples in the test set}
%\end{align*} 
%%The modulus operator denotes cardinality in the case of $X_{test}$ and absolute value otherwise. 
%%
%The true values of states are estimated by averaging 500 Monte Carlo roll-out samples.
%From a state $s \in X_{test}$, 1000 sample trajectories are generated, the sample returns computed
%and $V^*(s)$ set to the sample average of these 1000 returns. 
%%We also compare performance based on MSBE on new trajectories.
%%Not sure if this cite is applicable still
%%\citet{sun2015online} recently showed that any no-regret learning algorithm that minimizes the MSBE results in low prediction error. The batch 
%%MSBE 
%%MSRE conjugate-gradient optimization taken here is such a no-regret learner, and so MSBE reflects prediction error. Both 

\vspace{0.1cm}
\noindent
\textbf{Algorithms.}
We compare to using several fixed tile-coding (TC) representations. TC uses overlapping grids on the observation space. It is a sparse representation that is well known to perform well for \MC, \PW, and \AC. 
%Because both of the domains have $\xdim = 2$, the grid is two-dimensional.
We varied the granularity of the grid-size N and number of tilings D, 
where $D$ is the number of active features for each observation. 
%The configuration for each TC is $D-N$, %$D \times N \times N$.
The grid is either N$\times$N for \MC\ and \PW\ or N$^4$ for \AC.  
We explore (D=4, N=4), (D=4,N=8), (D=16,N=4), (D=16,N=8), (D=32,N=4), (D=32,N=8); a grid size of 16 performed poorly, and so is omitted.  
%We explore 4x4x4, 4x8x8, 16x4x4, 16x8x8, 32x4x4, 32x8x8; a grid size of 16 performed poorly, and so is omitted.  
For \MC\ and \PW\, the number of features
respectively  are 64, 256, 256, 1024, 512, 2048, then hashed to 1024 dimensions;
for \AC, the number of features are 1024, 16384, 4096, 65536, 8192, 131072, then hashed to 4096.
Both of these hashed sizes are much larger than our chosen $\rdim = 100$. 
%The configurations are of the form NxNxD, where NxN denotes the granularity of the tile coder, and D denotes the number of tilings, i.e., the sparsity. 
%The features in tile-coding algorithms are hashed to a linear space of 1024 dimensions. We ensure that the number of dimensions for TC is a power of 2, since it utilizes hashing algorithms optimized for such a configuration. 

For consistency, once the \scalg\ representation is learned, we use the same 
batch gradient descent update on the MSRE for all the algorithms, with line search to select step-sizes.  
%batch learning method is used to learn $\wvec$ for all algorithms, to separate the influence of the learning algorithm on performance. This batch learning method corresponds to the batch gradient descent update for the MSRE (see Algorithm 1),
%with line search to select step-sizes. 
The regularization weights $\regwgtb$ are chosen from $\{1^{-5}, \ldots, 1^{-1}, 0\}$, based on lowest cumulative error. For convenience, $\regwgtw$ is fixed to be the same as $\regwgtb$. 
For learning the \scalg\ representations, regularization parameters were chosen using 
5-fold cross-validation on 5000 training samples, with $\regwgt = 0.1$ fixed to give a reasonable level of sparsity.
This data is only used to learn the representation; for the learning curves, the weights are learned from scratch
in the same way they are learned for TC. 
%For \scalg, the regularizers and dimension $\rdim$ control the level of sparsity. To make this level comparable to the sparsities chosen for TC, we set the regularizers to produce the same sparsities--- 4, 16 and 32.
%The regularization weights are chosen for the specified sparsity by cross-validation as a pre-processing step. Specifically, for each configuration of parameters, we adopt 5-fold cross-validation on 5000 samples, calculating the average sparsity of representations and MSRE on testing sets. The combination of parameters that produces the minimum MSRE of representations within $20\%$ of the desired level of sparsity is chosen. For example, to obtain suitable parameters for sparsity 4, we choose parameters among those producing representations with average sparsities between 3.2 and 4.8.    
The dimension $\rdim = 100$ is set to be smaller than for tile coding, to investigate if \scalg\ can learn a more compact
sparse representation. 
%For \MC, $\rdim = 100$ and for \PW, we tested both $\rdim = 100$ and $\rdim = 500$. 
We tested unsupervised sparse coding, but the error was poor (approximately $10\times$ worse). 
%This illustrates the importance of a supervised sparse coding approach for policy evaluation. 
We discuss the differences between the representations learned by supervised and unsupervised sparse coding below. 

 \begin{figure*}[ht]
 \vspace{-0.2cm}
 \centering
 \includegraphics[scale=0.34]{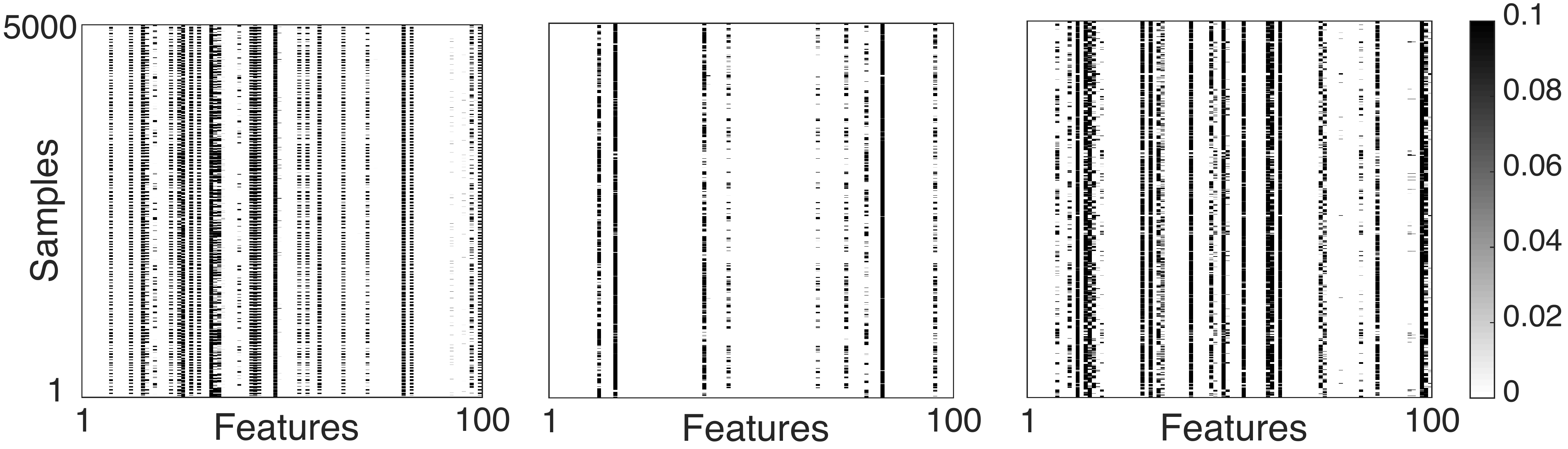}
%\minipage{0.37\textwidth}
%  \includegraphics[width=\linewidth]{figures/representations/sup_weight3_editted2}
% % \center{Eta sensitivity.}
%  \label{fig:sup_weights3}
%\endminipage\hfill
%\minipage{0.28\textwidth}%
%  \includegraphics[width=\linewidth]{figures/representations/us_reg_editted}
%  \label{fig:us_reg}
%\endminipage\hfill
%\minipage{0.35\textwidth}%
%  \includegraphics[width=\linewidth]{figures/representations/us_phip_editted}
%  \label{fig:us_phip}
%\endminipage
\caption{Learned representations $\repall$ for 5000 samples and $\rdim = 100$, respectively for \scalg, unsupervised sparse coding, and non-negative unsupervised sparse coding, in \PW. The representations learned for \MC\ and \AC\ have similar structure. 
%For \scalg\, the weight before dictionary learning from original representations to learned ones is set 3. 
The non-negative sparse coding has the additional constraint that each entry in $\repall$ is non-negative. 
The goal for this addition was to determine if further constraints could improve prediction accuracy for
unsupervised sparse coding; though the representation qualitatively looks more reasonable, prediction performance remained poor. 
%that learned representations cannot be negative. Each model reports 5000 learned representations(Y axis) with 100 features (X axis). The grey level denotes the absolute value of a feature of a representation. 
}
\label{fig:Representations}
\vspace{-0.4cm}
\end{figure*}

\vspace{0.1cm}
\noindent
\textbf{Learning curves.}
We first demonstrate learning with increasing number of samples, in Figure \ref{fig:LC}. The weights are recomputed using the entire batch
up to the given number of samples.
% with regularization selected based on lowest cumulative error. 
%The graphs show early learning with 1000 samples, recomputing weights every 50 samples; the final error after 5000 samples
%is reported in brackets. 
%we report
%the final error when training on 5000 samples within brackets for each algorithm. 

Across domains, \scalg\ results in faster learning and, in \MC\ and \AC, obtains lowest final error. 
Matching the performance of TC is meaningful, as TC is well-understood and optimized for these domains.
For \AC, it's clear a larger TC is needed resulting in relatively poor performance, whereas \scalg\ can still
perform well with a compact, learned sparse representation. 
These learning curves provide some insight that we can learn effective sparse representations with \scalg, 
but also raise some questions. One issue is that \scalg\ is not as effective in \PW\ as some of the TC representations,
namely 4-4 and 16-4. The reason for this appears to be that we optimize MSRE to obtain the representation, which is a surrogate
for the MAPVE. When measuring MSRE instead of MAPVE on the test data, \scalg\ consistently outperforms TC.
Optimizing both the representation and weights according to MSRE may have overfitting issues;
extensions to MSBE or BE, or improvements in selecting regularization parameters, may alleviate this issue.

\vspace{0.1cm}
\noindent
\textbf{Learned representations.}
We also examine the learned representations, both for unsupervised sparse coding
and \scalg, shown in Figure \ref{fig:Representations}. We draw two conclusions from these
results: the structure in the observations is not sufficient for unsupervised sparse coding,
and the combination of supervised and unsupervised losses sufficiently constrain the space
to obtain discriminative representations. 
For these two-dimensional and four-dimensional observations, 
it is relatively easy to reconstruct the observations by using only a small subset of dictionary atoms (row vectors of $\rparams$ in equation \eqref{eq_scope}). The unsupervised representations, even with additional non-negativity constraints to narrow the search space,
 are less distributed, with darker and thicker blocks, and more frequently pick less features.  
 For the supervised sparse coding representation, however, the sparsity pattern is smoother and more distributed: more features are selected by at least one sample, but the level of sparsity is similar. 
 We further verified the utility of supervised sparse coding, by only optimizing the supervised loss (MSRE), without including the unsupervised loss;
 the resulting representations looked similar to the purely unsupervised representations. 
 The combination of the two losses, therefore, much more effectively constrains or regularizes
 the space of feasible representations and improves discriminative power.

The learning demonstrated for \scalg\ here is under ideal conditions. This was intentionally chosen
to focus on the question: can we learn effective sparse representations 
using the \scalg\ objective?
With the promising results here, future work needs to investigate the utility of jointly estimating the representation and learning the value function, as well as providing incremental algorithms for learning the representations and setting the regularization parameters.

\section{Conclusion}

In this work, we investigated sparse coding for policy evaluation in reinforcement learning.
We proposed a supervised sparse coding objective, for joint estimation
of the dictionary, sparse representation and value function weights. 
We provided a simple algorithm that uses alternating minimization
on these variables, and proved that this simple and easy-to-use approach
is principled. 
%In particular, we avoided the need for the typical careful initialization 
%schemes for sparse coding by using $\Gamma$-convergence to extend recent results for
%proved that for this objective, all local minima
%are in fact global minima, ensuring that when such a block coordinate descent strategy 
%converges to local minima, it is in fact the global solution.
We finally demonstrate results on three benchmark domains, Mountain Car, Puddle World and Acrobot,
against a variety of configurations for tile coding.
% which is a sparse representation that is well known to perform well in these domains.

This paper provides a new view of using dictionary learning techniques from machine learning
in reinforcement learning. It lays a theoretical and empirical foundation for further investigating sparse coding, and other dictionary learning approaches, for policy evaluation and suggests that they show some promise. 
Formalizing representation learning as a dictionary learning problem
facilitates extending recent and upcoming advances in
unsupervised learning to the reinforcement learning setting.
For example, though we considered a batch gradient descent approach for this first investigation,
the sparse coding objective is amenable to incremental estimation, with several works
investigating effective stochastic gradient
descent algorithms \citep{mairal2009supervised,mairal2010online,le2016global}.
The generality of the approach and easy to understand optimization
make it a promising direction for representation learning in reinforcement learning.

{
%\small
\footnotesize
\bibliographystyle{named}
\bibliography{paper}
}

\end{document}